\documentclass[11pt,a4paper]{article}

\usepackage[utf8]{inputenc}
\usepackage[T1]{fontenc}
\usepackage{lmodern}
\usepackage{microtype}

\usepackage[margin=1in]{geometry}

\usepackage{amsmath,amssymb}
\usepackage{amsthm}
\theoremstyle{definition}
\newtheorem{definition}{Definition}
\theoremstyle{plain}
\newtheorem{proposition}{Proposition}

\usepackage{graphicx}
\graphicspath{{figures/}}
\usepackage[font=small,labelfont=bf]{caption}
\usepackage{subcaption}

\usepackage{booktabs}
\usepackage{multirow}
\usepackage{array}
\newcolumntype{C}[1]{>{\centering\arraybackslash}p{#1}}

\usepackage[dvipsnames]{xcolor}
\usepackage[colorlinks=true,linkcolor=MidnightBlue,citecolor=OliveGreen,urlcolor=BrickRed]{hyperref}

\usepackage[numbers,sort&compress]{natbib}

\usepackage{enumitem}
\usepackage{url}
\usepackage{float}

\newcommand{\E}{\mathbb{E}}
\newcommand{\real}{\mathrm{real}}
\newcommand{\simu}{\mathrm{sim}}

\title{\textbf{Simulated Customers Never Walk Away:\\
Decision Fidelity of LLM User Simulators\\
Measured Against Real Purchase Outcomes}}

\author{
  Liang Chen
}
\date{June 2026}

\begin{document}
\maketitle

\begin{abstract}
LLM-as-user-simulation has quietly become core infrastructure for
conversational AI: agent benchmarks ($\tau$-bench), training pipelines, and a
fast-growing body of fidelity studies all rely on LLMs role-playing the human
side of the dialogue. Existing validation frameworks measure
\emph{communicative fidelity}---whether simulators talk like humans---against
ground truth produced by paid participants role-playing assigned goals. We
argue this paradigm has a structural blind spot: when the user's goal is
assigned, the user's \emph{willingness} is exogenous, so no existing framework
can test whether simulators make \emph{decisions} like real users whose
motivation is endogenous, latent, and decaying. We introduce \emph{decision
fidelity}---whether a simulated population reproduces the decision-state
dynamics of real users facing real, consequential choices---and measure it on
a unique testbed: 2{,}790 production conversations between an LLM sales agent
and real customers, including 793 with verified payment outcomes. Using a
teacher-forced probe protocol that holds context and measurement instrument
fixed, we find a systematic, outcome-correlated failure we call the
\emph{disengagement deficit}: simulators reproduce eventual buyers almost
exactly (depth bias $+0.09$) but inflate eventual non-buyers toward the
purchase frame (depth bias $+0.40$; group contrast $d{=}0.38$,
$p{<}0.001$), halving their expressed resistance (25.1\%$\to$13.5\%) and
nearly doubling deliberation (21.9\%$\to$40.1\%) while fabricating no
purchases. The deficit replicates across model families (DeepSeek simulator:
$d{=}0.41$, $p{=}0.002$) and resists the obvious fix: explicitly instructing the
simulator that it may disengage cuts \emph{marginal} bias five-fold but barely
moves the outcome-conditioned contrast ($d{=}0.34$, $p{=}0.008$), merely
relocating the error onto eventual buyers---decision fidelity is not a prompting
oversight. Real non-buyers
say ``not now'' and stop; simulated non-buyers ask about price. Evaluating or
training sales and persuasion agents against such simulators systematically
overstates funnel progress exactly where it matters most---the customers who
walk away.
\end{abstract}


\section{Introduction}
\label{sec:intro}

Evaluating a conversational agent requires a counterparty. Because real humans
are expensive, slow, and hard to standardize, the field has converged on a
pragmatic substitute: \emph{LLM-as-user-simulation}, in which a language model
role-plays the human side of the dialogue. This substitution is no longer a
niche convenience. It underlies prominent agent benchmarks such as
$\tau$-bench~\cite{yao2024taubench}, supplies the interactive environments used
to train and tune dialogue agents, and has spawned a dedicated research thread
on \emph{simulator fidelity}~\cite{realusersim2026,convapparel2026,sim2real2026}.
The premise of all of this work is the same: \emph{if a simulator behaves
enough like a real user, conclusions drawn against it transfer to
reality.}

The fidelity literature has rightly focused on whether simulators
\emph{communicate} like humans. Recent benchmarks measure stylistic realism,
persona consistency, verbosity, information-disclosure patterns, and emotional
reactions, and report a persistent ``realism
gap''~\cite{convapparel2026,sim2real2026}: simulators are too verbose, too
uniformly polite, too patient, and too cooperative. We take this work as our
point of departure and ask a question it cannot answer.

\paragraph{The blind spot.} Every existing fidelity benchmark validates against
ground truth produced by \emph{paid participants role-playing an assigned
goal}---book this flight, return this item, find this product. When the goal is
assigned, the user's \emph{willingness to act is exogenous}: it is handed to the
participant as part of the task. Consequently, no existing framework can test
whether a simulator reproduces the part of human behavior that matters most in
the highest-stakes deployments: the \emph{endogenous, latent, and decaying
willingness} of a real person deciding whether to commit. A real customer can
lose interest, stall, deflect, and quietly disappear---and whether they do is
\emph{the outcome the agent exists to influence}. A role-player told to buy a
flight never genuinely decides not to.

We call the missing construct \emph{decision fidelity}: whether a simulated
population reproduces the \emph{decision-state dynamics}---how willingness,
resistance, deliberation, and disengagement evolve under persuasion---of real
users facing real, consequential choices. Communicative fidelity asks whether
the simulator \emph{sounds} human; decision fidelity asks whether it
\emph{decides} like one. For evaluating sales, fundraising, negotiation, and
other persuasion agents, decision fidelity is the property that actually
determines whether simulator-based conclusions transfer---and it is precisely
the property no prior benchmark can measure, because no prior benchmark has
real outcomes.

\paragraph{A testbed with real stakes.} We measure decision fidelity using a
resource the fidelity literature lacks: \emph{ZhenaiSales}, 2{,}790 production
conversations between a deployed LLM sales agent and real customers of a Chinese
relationship-matchmaking service, of which 793 end in a \emph{verified payment}
and the remainder do not. These are not role-players; their willingness is real
and their decisions cost real money. This lets us pose decision fidelity as a
measurement problem with ground truth---and, crucially, sidesteps the
constraint that doomed our earlier attempts to make causal claims from
observational dialogue data: fidelity is a question about \emph{distributional
match}, for which observational data with outcomes is sufficient.

\paragraph{Method.} We compare a simulator to reality with a teacher-forced
probe protocol that controls the two confounds that would otherwise contaminate
the comparison. (i) \emph{Context is held fixed}: at a probe point we feed the
simulator the \emph{real} conversation prefix and ask only for the next user
turn, so the simulator and the real user respond to identical histories.
(ii) \emph{The measurement instrument is held fixed}: an LLM perceiver maps both
the real and the simulated turn to the same decision-state schema (engagement
stage, emotion, blocker), so differences reflect behavior, not labeling. We
then compare real and simulated decision-state distributions, both marginally
and---this is the key move---conditioned on the user's \emph{eventual real
outcome}.

\paragraph{Finding: the disengagement deficit.} Simulators reproduce eventual
\emph{buyers} almost exactly (mean engagement-depth bias $+0.09$) but
systematically inflate eventual \emph{non-buyers} toward the purchase frame
(bias $+0.40$). The gap is significant and outcome-specific (group contrast
$d{=}0.38$, permutation $p{<}0.001$, $n{=}374$ conversations): a simulator's
infidelity is concentrated entirely on the customers who, in reality, do not
buy. Mechanistically, the simulator \emph{halves} the resistance real
non-buyers express (25.1\%$\to$13.5\%) and \emph{nearly doubles} their
deliberation (21.9\%$\to$40.1\%), while fabricating essentially no purchases
(``deciding'' stage unchanged). It does not invent fake sales; it inflates the
\emph{middle of the funnel}, manufacturing a population of interested
deliberators where reality contains people who say ``I'm busy,'' ``no thanks,''
and then go silent. We name this the \emph{disengagement deficit}: simulated
customers never walk away.

\paragraph{Generality.} The deficit is not an artifact of one model, one prompt,
or one judge. It replicates when the simulator is swapped to a different model
family (DeepSeek-V4-Flash: $d{=}0.41$, $p{=}0.002$) and strengthens when the
decision-state instrument is swapped to a different model family
($d{=}0.42$, $p{=}0.015$), ruling out a judge artifact. Most tellingly, explicitly instructing
the simulator that it may be disinterested, perfunctory, or impatient cuts its
\emph{marginal} bias five-fold yet leaves the outcome-conditioned contrast
almost untouched ($d{=}0.34$, $p{=}0.008$)---it learns to perform disengagement
but not to direct it at the right people, over-correcting buyers instead. This
points to a paradigm-level cause rather than a fixable prompt: instruction
tuning rewards being a helpful, engaged interlocutor, whereas real
willingness-decay---deflection, stalling, disappearance---is largely absent from
alignment data.

\paragraph{Why it matters.} An agent evaluated against a simulator that cannot
disengage is graded on an easy population: it rarely meets hard refusal, and it
can never be tested on the skill that most determines real sales
performance---recognizing and responding to a customer who is leaving. We show
(\S\ref{sec:consequence}) that this biases the apparent value of sales
tactics: after a pitch, simulated customers deliberate 75\% of the time versus
45\% for real ones, so pressure and pitching look far more effective against
simulated customers, who deliberate rather than disengage, than against real
ones. Any pipeline that
evaluates---or worse, reinforcement-trains---a persuasion agent against such
simulators optimizes a reward landscape that is wrong precisely where the money
is.

\paragraph{Contributions.}
\begin{enumerate}
  \item We articulate \textbf{decision fidelity} as a distinct, measurable
  property of user simulators, and argue it---not communicative
  fidelity---governs whether simulator-based evaluation of persuasion agents
  transfers to reality (\S\ref{sec:formulation}).
  \item We introduce a \textbf{teacher-forced, instrument-controlled protocol}
  for measuring decision fidelity against real consequential outcomes, built on
  \textbf{ZhenaiSales}, the first user-simulation testbed grounded in verified
  purchase decisions; we release the protocol, instruments, and derived
  statistics, with anonymized conversation data subject to privacy review
  (\S\ref{sec:method},~\S\ref{sec:experiments}).
  \item We document the \textbf{disengagement deficit}---a systematic,
  outcome-correlated, cross-model failure of LLM user simulators to reproduce
  real willingness-decay---and trace its mechanism and its consequences for
  agent evaluation and training (\S\ref{sec:experiments},~\S\ref{sec:analysis}).
\end{enumerate}

\section{Related Work}
\label{sec:related}

\subsection{LLM User Simulation for Agent Evaluation and Training}
Simulating the user has long been a route to scalable evaluation and training
of dialogue systems~\cite{schatzmann2007,li2016user}. LLMs turned earlier
rule-based simulators into fluent generative ones used across search, recommendation,
task-oriented dialogue, and tool agents~\cite{realusersim2026,sim2real2026}.
The paradigm is now load-bearing: $\tau$-bench~\cite{yao2024taubench} evaluates
tool agents through dynamic conversations with an LLM-simulated user, and its
\texttt{pass\textasciicircum k} reliability metric---and the leaderboards built
on it---inherit whatever biases the simulator carries. Inspection of the
$\tau$-bench user simulator is instructive: the user is driven by an assigned
\emph{instruction}, told to reveal information gradually and not to hallucinate,
and emits a \texttt{\#\#\#STOP\#\#\#} token \emph{only when the instruction goal
is satisfied}. The simulated user therefore cannot decline the goal; willingness
is a fixed input, not a behavior. Our work asks what is lost by that assumption.

\subsection{Measuring Simulator Fidelity}
A 2026 wave of work measures how far simulators diverge from humans.
\emph{RealUserSim}~\cite{realusersim2026} grounds simulators in 7{,}275 WildChat
personas and scores fidelity along five communicative dimensions (persona/affect,
linguistic style, competency, information flow, pacing) via a paired-trajectory
Turing test; it explicitly studies ``how users communicate rather than what they
\emph{decide}.'' RealUserSim also coins the ``Formalism Ceiling'' (6--8\% baseline
style-match) and identifies ``Directive Amplification''---identical behavioral directives
produce dramatically different responses across models (roleplay markers 0--19.4\%),
making cross-model fidelity comparisons unreliable.
Google's \emph{ConvApparel}~\cite{convapparel2026} collects
4{,}146 human--AI apparel-shopping conversations with both ``good'' and ``bad''
recommenders and first-person satisfaction/frustration annotations, and validates
simulators by population-level statistical alignment, a learned human-likeness
discriminator, and counterfactual generalization; notably its prompted baseline
\emph{instructs} the simulator to ``quit if overly annoyed,'' yet a realism gap
persists. \emph{Mind the Sim2Real Gap}~\cite{sim2real2026} runs 31 simulators on
$\tau$-bench against 451 human role-players and finds simulators are
``excessively cooperative, stylistically uniform, and lack realistic frustration,''
creating an ``easy mode'' that inflates agent success rates.
Independently, \emph{Lost in Simulation}~\cite{seshadri2026lost} finds systematic
miscalibration (ECE\,=\,15.1) with task-difficulty-dependent inflation, corroborating the
cooperative-bias findings of the Sim2Real work from a calibration perspective.

These studies establish that simulators are too agreeable---but they diagnose it
as a \emph{communicative/affective} gap measured against role-players. Two
structural limits follow. First, their ground truth is paid participants
executing assigned goals (ConvApparel pays contractors; Sim2Real uses role-play;
RealUserSim uses chit-chat personas), so the participant's \emph{willingness} is
exogenous and a simulator can match every communicative statistic while getting
the \emph{decision} wrong. Second, ``cooperativeness'' is treated as a stylistic
trait, not as an \emph{outcome-conditioned} distortion. We complement this
literature by (i) defining fidelity over decision states rather than
communication, (ii) grounding it in \emph{verified purchase outcomes} rather than
role-play, and (iii) showing the distortion is not uniform but concentrated on
the eventual non-buyers---a structure invisible to outcome-free benchmarks.

\subsection{Cooperativeness, Sycophancy, and Willingness}
LLMs are known to be sycophantic and over-agreeable~\cite{sharma2024sycophancy},
and user simulators inherit this~\cite{sim2real2026}. Work on non-collaborative
and willingness-aware dialogue explicitly models users who resist or may not
want to engage~\cite{hentona2025userwillingness,dutt-etal-2021-resper}.
Taubenfeld et al.~\cite{taubenfeld2026stated} evaluate 25 LLMs and find that
self-reported disposition scores fail to predict revealed behavioral rankings,
implying that persona-calibrated simulators may not produce behaviorally
faithful simulation even when their declared preferences match the target user.
Our finding sharpens the sycophancy account: the deficit is not that simulators are globally
too nice, but that they cannot represent \emph{declining} willingness in the
specific people whose willingness, in reality, declines---a failure with sign,
location, and consequences, not just a politeness offset.

\subsection{LLM-as-Judge and Using LLM Annotations for Inference}
Our measurement instrument is an LLM perceiver, situating us alongside
LLM-as-judge work and its known biases~\cite{zheng2023judging,llm-judge-survey},
and alongside methods for valid downstream inference from imperfect LLM
annotations~\cite{egami2023dsl}. We control for instrument bias by (a) holding
the instrument fixed across the real/simulated comparison so labeling error
cancels in the contrast, and (b) swapping the instrument to a second model family
as a robustness check (\S\ref{sec:experiments}).
Hullman et al.~\cite{hullman2026validity} provide a formal statistical framework
showing that even at 90\% LLM prediction accuracy, correlated errors can bias
downstream estimates by ${\sim}$30\%, further motivating our instrument-swap design.
Schessl~\cite{schessl2026autocorrelation}
warns that turn-level LLM conversation findings can be inflated by autocorrelation;
our primary endpoint is computed at the \emph{conversation} level precisely to
avoid this.

\subsection{AI for Sales and Persuasion}
Persuasion and sales dialogue is an active
area~\cite{wang-etal-2019-persuasion,zhang2025aisalesman,nandakishor2025salesrlagent},
and recent evidence questions how much per-turn \emph{strategy} explains of
outcomes~\cite{petrova2026persuasion}. This motivates our consequence analysis:
if a simulator suppresses resistance after sales pressure, it will reward
pressure tactics that real customers ignore or punish, corrupting exactly the
signal a strategy-optimizing agent learns from.

\section{Decision Fidelity: Problem Formulation}
\label{sec:formulation}

\subsection{Setup}
A persuasion dialogue is a sequence of alternating turns
$c = (a_1, u_1, a_2, u_2, \dots)$ between an agent ($a_t$) and a user ($u_t$),
ending in a binary, consequential outcome $y \in \{0,1\}$ (here, payment). A
\emph{user simulator} $S$ is a conditional generator that, given a profile $\pi$
and a history $h_{t} = (a_1,u_1,\dots,a_t)$, produces the next user turn
$\hat u_t \sim S(\cdot \mid \pi, h_t)$. A \emph{decision-state instrument}
$\Phi$ maps a user turn in context to a latent state
$\Phi(u_t, h_t) = s_t \in \mathcal{S}$ capturing where the user sits on the path
to commitment (e.g., exploring, engaging, considering, deciding, resisting).

\subsection{Communicative vs.\ decision fidelity}
Prior fidelity work evaluates $S$ by how closely the \emph{surface form} of
$\hat u_t$ matches human turns---style, length, politeness, dialog
acts~\cite{realusersim2026,convapparel2026,sim2real2026}. Formally this asks
whether the distribution of some communicative featurization $g(\hat u_t)$
matches $g(u_t)$. We instead ask whether $S$ reproduces the \emph{decision
dynamics} a real population exhibits, i.e.\ whether the induced distribution over
decision states matches reality.

\begin{definition}[Decision fidelity]
A simulator $S$ is \emph{decision-faithful} with respect to instrument $\Phi$ and
a real population $\mathcal{D}$ if, for histories $h$ drawn from $\mathcal{D}$,
the simulated and real next-state distributions agree:
$\Phi(\hat u \mid h) \stackrel{d}{=} \Phi(u \mid h)$, where $\hat u \sim S(\cdot
\mid \pi, h)$ and $u$ is the real continuation. \emph{Conditional} decision
fidelity additionally requires agreement within each outcome stratum:
$\Phi(\hat u \mid h, y) \stackrel{d}{=} \Phi(u \mid h, y)$ for $y \in \{0,1\}$.
\end{definition}

Conditional fidelity is the operative requirement. A simulator can match the
\emph{marginal} state distribution---looking realistic on average---while
systematically misassigning states \emph{within} the buyer and non-buyer strata
in compensating directions. Because an agent's job is precisely to move
different users differently, marginal realism is insufficient; the simulator
must be faithful where the populations diverge.

\subsection{A scalar endpoint: outcome-conditioned depth bias}
\label{sec:endpoint}
To obtain a powered, pre-registerable test we reduce $\mathcal{S}$ to an ordinal
\emph{engagement depth} $\delta:\mathcal{S}\to\mathbb{Z}$ measuring proximity to
purchase, with resisting$<$exploring$<$engaging$<$considering$<$deciding (we use
$\{$resisting$=0$, exploring$=1$, engaging$=2$, considering$=3$, deciding$=4\}$;
robustness to alternative codings is reported in
Appendix~\ref{app:robustness}). For a probe at
history $h$ with real continuation $u$ and simulated continuation $\hat u$, the
\emph{depth bias} is
\[
  D(h) \;=\; \delta\big(\Phi(\hat u, h)\big) - \delta\big(\Phi(u, h)\big).
\]
We aggregate $D$ to the conversation level (mean over a conversation's probes) to
avoid within-conversation autocorrelation~\cite{schessl2026autocorrelation}, and
define the \textbf{primary endpoint} as the contrast between outcome strata,
\[
  \Delta \;=\; \E_{y=0}\!\big[\bar D\big] \;-\; \E_{y=1}\!\big[\bar D\big].
\]
Perfect conditional decision fidelity implies $\E_{y=0}[\bar D] =
\E_{y=1}[\bar D] = 0$, hence $\Delta = 0$. A positive $\Delta$ means the
simulator inflates engagement depth \emph{more} for eventual non-buyers than for
buyers---the signature of a population that cannot disengage. We test $\Delta>0$
by a conversation-level permutation test and report Cohen's $d$.

\begin{proposition}[Invariance properties of the endpoint]
\label{prop:invariance}
Under the teacher-forced protocol the real and simulated turns share an identical
history $h$ and are scored by the same instrument $\Phi$. Therefore (i) any
additive, history-dependent labeling bias $b(h)$ of the instrument cancels
within each probe in $D(h)$, and (ii) any outcome-\emph{independent} shift of
the simulator (e.g.\ a generic style or stance offset) contributes equally to
both strata and cancels in $\Delta$. Only \emph{outcome-correlated} differences
survive. (Sketch: $\Delta$ differences out $b(h)$ within probes, then
differences across strata.)
\end{proposition}

Proposition~\ref{prop:invariance} is why $\Delta$, rather than a raw agreement
rate, is our headline: it is robust by construction to the two confounds---judge
miscalibration on contexts and generic simulator style---that the prior
literature treats as the whole story. One channel is \emph{not} covered by the
cancellation argument: a judge bias triggered by \emph{message style} that
correlates with outcome through real users' styles (e.g.\ over-reading curt real
messages as resistance). We address this residual channel empirically, by
swapping the instrument to a different model family and observing the same
contrast (\S\ref{sec:experiments}), and by inspecting probes qualitatively
(Table~\ref{tab:examples}).

\section{Method: Teacher-Forced Decision-Fidelity Probing}
\label{sec:method}

Our protocol (Figure~\ref{fig:protocol}) measures $\Delta$
(\S\ref{sec:endpoint}) while controlling the two confounds that
Proposition~\ref{prop:invariance} requires us to neutralize: divergent context
and divergent measurement.

\begin{figure}[t]
\centering
\includegraphics[width=\linewidth]{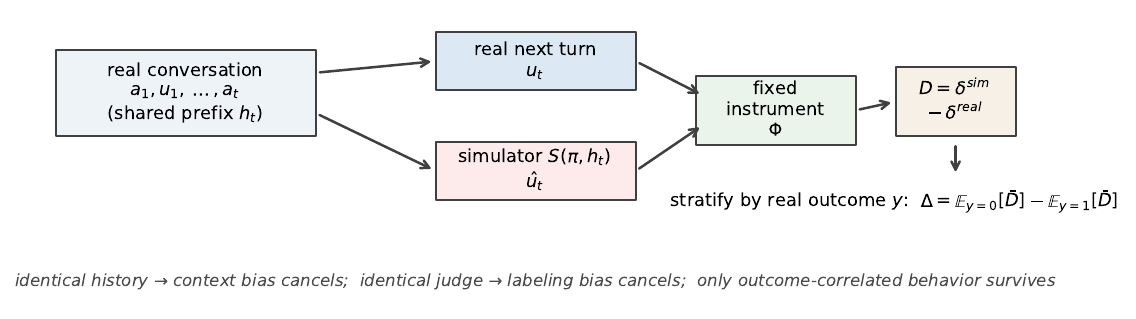}
\caption{The teacher-forced decision-fidelity protocol. At each probe, the
real user and the simulator continue the \emph{identical} real prefix $h_t$,
and both continuations are scored by the \emph{same} instrument $\Phi$; the
paired depth bias $D$ is aggregated per conversation and contrasted across
real-outcome strata. Context-dependent and judge-dependent biases cancel by
construction (Prop.~\ref{prop:invariance}); only outcome-correlated behavioral
differences survive in $\Delta$.}
\label{fig:protocol}
\end{figure}

\subsection{Teacher-forced probes}
For each real conversation $c$ with outcome $y$, we select probe turns at
$\{30\%, 60\%, 90\%\}$ of its user turns (skipping the opening greeting). At a
probe over user turn $u_t$ with real prefix $h_t$:
\begin{enumerate}
  \item \textbf{Real branch.} Score the real turn: $s^\real = \Phi(u_t, h_t)$.
  \item \textbf{Sim branch.} Give the simulator the \emph{same} real prefix
  $h_t$ plus the user's profile $\pi$ and ask for the next user turn:
  $\hat u_t \sim S(\cdot\mid\pi, h_t)$; score it under the same instrument:
  $s^\simu = \Phi(\hat u_t, h_t)$.
\end{enumerate}
Because both branches condition on the \emph{identical} real history, the
comparison isolates the user model from the agent and from any cumulative
divergence a free-running rollout would introduce. Teacher forcing is the
standard device for fidelity probing~\cite{convapparel2026}; we use it to obtain
paired $(s^\real, s^\simu)$ at matched histories, the unit of
Proposition~\ref{prop:invariance}.

\subsection{The decision-state instrument $\Phi$}
$\Phi$ is an LLM perceiver prompted to map a user turn, in its local context, to
a structured state: an \emph{engagement stage}
(exploring/engaging/considering/deciding/resisting), an \emph{emotion}
(positive/neutral/hesitant/negative), and a \emph{blocker} (none/price/trust/
capability/timing/external/prior-failure). The prompt enforces constrained JSON
output; it is reproduced in Appendix~\ref{app:prompt}. The instrument is applied
\emph{causally}---scoring turn $t$ from context up to $t$ only---so that scores
do not leak future information; we verified separately that causal vs.\ full-context
labeling does not produce outcome-correlated differences, ruling out hindsight
leakage in $\Phi$ itself.\footnote{In a format-controlled check, varying only the
instrument's visible window (prefix vs.\ full conversation) yielded no
outcome-correlated label shift ($d{=}{-}0.04$, $p{=}0.62$), so $\Phi$ does not
itself encode the outcome.} The same fixed $\Phi$ scores both branches, so its
biases cancel in $\Delta$.

\subsection{The simulator $S$}
Our primary simulator is a profile-conditioned prompted parent: it receives a
natural-language rendering of the user's real demographic profile (child's age,
education, occupation, income; parent's location; housing/vehicle status) and is
asked to continue the conversation as that parent, in colloquial, brief,
chat-style Chinese. This is the standard, competitive ``prompted persona''
simulator of the fidelity
literature~\cite{convapparel2026,realusersim2026}. To probe generality we vary
$S$ along three axes (\S\ref{sec:experiments}): the backbone model (Claude vs.\
DeepSeek-V4-Flash, spanning Western and Chinese model families); and an
\emph{instructed} variant whose prompt explicitly licenses the full range of real
behavior---disinterest, perfunctory replies, impatience, topic avoidance,
non-response---mirroring ConvApparel's ``quit if annoyed'' baseline and testing
whether the deficit is a mere prompt-engineering oversight.

\subsection{Estimation}
We compute conversation-level $\bar D$, the stratum means $\E_{y}[\bar D]$, the
contrast $\Delta$, Cohen's $d$, and a one-sided conversation-level permutation
test for $\Delta>0$ (20{,}000 permutations). We also report the full real-vs-sim
stage distributions within each stratum to expose the mechanism, and an
action-conditioned analysis (\S\ref{sec:experiments}) that labels the agent move
preceding each probe to test how the deficit distorts the agent's apparent reward
for specific tactics.

\section{Experiments}
\label{sec:experiments}

\subsection{ZhenaiSales: a testbed with real outcomes}
\label{sec:data}
ZhenaiSales comprises 2{,}790 production conversations between a deployed LLM
sales agent (``red\-/matchmaker manager'') and real parent customers of a Chinese
relationship-matchmaking platform, extracted under a fixed data
specification.\footnote{Personalized agent--user text messages; paid$=$verified
payment order (action code 3); 20 internal test users excluded.} Of these, 793
are \emph{converted} (the parent made a verified payment) and 1{,}997 are
non-converted; 98\% carry a structured demographic profile. Conversations have a
median of 16 turns (6--80). Critically, converted conversations are
\textbf{truncated at the first payment timestamp}, removing post-purchase service
chat that would otherwise leak the outcome (mean 19 turns dropped per buyer);
this makes pre-decision dynamics the only signal. Unlike role-play corpora, every
label here reflects a real person spending---or declining to spend---real money.
For the fidelity experiments we draw outcome-balanced samples and reuse the real
decision-state labels across simulator conditions.

\subsection{Primary result: the disengagement deficit}
\label{sec:primary}
We probe $n{=}374$ conversations (181 non-converted, 193 converted; 1{,}109
matched probes) with the primary prompted-persona simulator (Claude backbone) and
the fixed instrument $\Phi$. Table~\ref{tab:primary} and
Figure~\ref{fig:deficit} report the endpoint.

\begin{figure}[t]
\centering
\includegraphics[width=\linewidth]{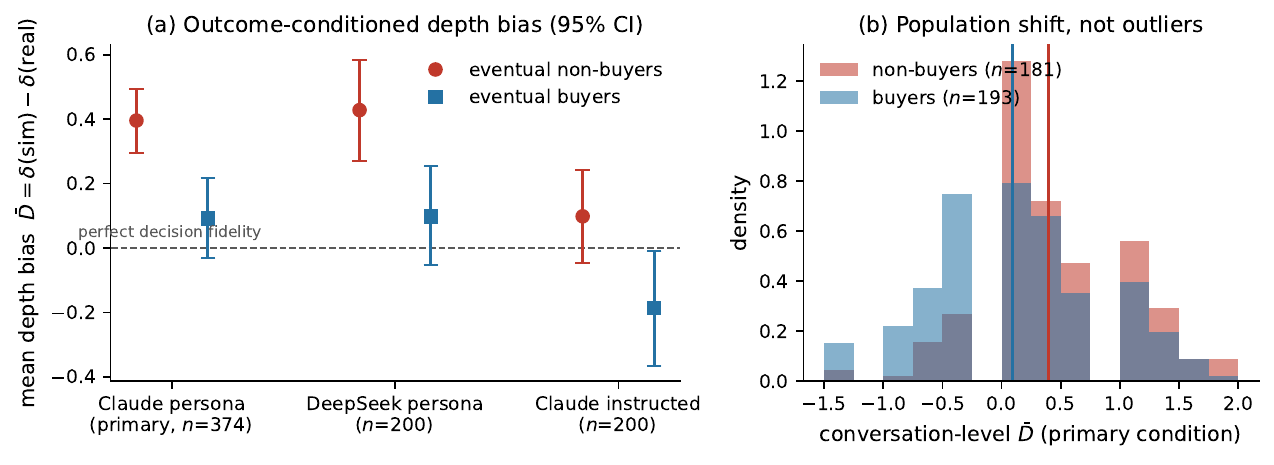}
\caption{The disengagement deficit. \textbf{(a)} Conversation-level mean depth
bias by real-outcome stratum with bootstrap 95\% CIs. Both persona simulators
track eventual buyers ($\approx$0) while inflating eventual non-buyers
($\approx{+}0.4$). Explicit disengagement instructions (right) repair the
non-buyer stratum but push buyers \emph{negative}---the error is relocated, not
removed, and the stratum gap (the deficit) persists. \textbf{(b)} The
distribution of conversation-level bias under the primary simulator: a
population-wide shift of non-buyers toward inflated engagement, not a tail of
outliers.}
\label{fig:deficit}
\end{figure}

\begin{table}[t]
\centering
\caption{Primary endpoint: outcome-conditioned engagement-depth bias
$\bar D = \delta(\mathrm{sim})-\delta(\mathrm{real})$, aggregated per
conversation. $\Delta$ is the non-buyer$-$buyer contrast; perfect decision
fidelity implies all entries $=0$. Permutation test is one-sided for
$\Delta>0$. Brackets show bootstrap 95\% CIs.}
\label{tab:primary}
\newcommand{\ci}[1]{{\scriptsize[#1]}}
\begin{tabular}{lcccc}
\toprule
Simulator & $\E_{y=0}[\bar D]$ & $\E_{y=1}[\bar D]$ & $\Delta$ & $d$ ($p$) \\
 & (non-buyers) & (buyers) & & \\
\midrule
Claude persona & $+0.396$ & $+0.092$ & $+0.304$ & $0.38$ ($0.0002$) \\
\quad ($n{=}374$) & \ci{0.30, 0.49} & \ci{$-$0.04, 0.22} & \ci{0.14, 0.46} & \ci{0.18, 0.59} \\[3pt]
DeepSeek persona & $+0.428$ & $+0.098$ & $+0.330$ & $0.41$ ($0.002$) \\
\quad ($n{=}200$) & \ci{0.27, 0.59} & \ci{$-$0.06, 0.26} & \ci{0.11, 0.56} & \ci{0.12, 0.71} \\[3pt]
Claude, instructed & $+0.098$ & $-0.187$ & $+0.285$ & $0.34$ ($0.008$) \\
\quad ($n{=}200$) & \ci{$-$0.05, 0.24} & \ci{$-$0.36, $-$0.01} & \ci{0.06, 0.51} & \ci{0.06, 0.63} \\
\bottomrule
\end{tabular}
\end{table}

Two facts stand out. First, the simulator reproduces eventual \emph{buyers}
nearly perfectly ($\bar D = +0.09$, a tenth of a stage), so it is not globally
biased. Second, it inflates eventual \emph{non-buyers} four-fold more
($\bar D = +0.40$), yielding a significant, outcome-specific contrast
($\Delta=+0.30$, $d=0.38$, $p<0.001$). The infidelity lives entirely in the
population the simulator gets wrong: the people who, in reality, do not buy.

\subsection{Mechanism: suppressed resistance, inflated deliberation}
\label{sec:mechanism}
Table~\ref{tab:mechanism} decomposes the bias into stage shifts. For non-buyers
the simulator \emph{halves} expressed resistance (25.1\%$\to$13.5\%) and
\emph{nearly doubles} deliberation (21.9\%$\to$40.1\%), while leaving the
``deciding'' (imminent-purchase) rate essentially unchanged (4.3\%$\to$4.3\%).
The simulator does not fabricate purchases; it manufactures \emph{interested
deliberators} out of people who, in reality, withdraw. Buyers show only minor
shifts. Qualitatively (Table~\ref{tab:examples}): where a real non-buyer says
``busy,'' ``no thanks,'' or sends a brand-mismatched link and goes quiet, the
simulator answers ``how much is it?'' or ``let me think it over, the child's
matter can't really be delayed.''

\begin{table}[t]
\centering
\caption{Stage-distribution shift (real$\to$sim) by outcome stratum, primary
simulator. The deficit is the non-buyer rows: resistance suppressed,
deliberation inflated, purchase rate unchanged.}
\label{tab:mechanism}
\begin{tabular}{llccc}
\toprule
Stratum & Stage & real & sim & $\Delta$pp \\
\midrule
\multirow{3}{*}{Non-buyers} & resisting   & 25.1\% & 13.5\% & $-11.6$ \\
                            & considering & 21.9\% & 40.1\% & $+18.2$ \\
                            & deciding    &  4.3\% &  4.3\% & $\phantom{+}0.0$ \\
\midrule
\multirow{2}{*}{Buyers}     & resisting   & 16.0\% & 13.6\% & $-2.4$ \\
                            & considering & 33.7\% & 46.6\% & $+12.9$ \\
\bottomrule
\end{tabular}
\end{table}

\begin{table}[t]
\centering
\caption{Representative non-buyer probes (translated). Real parents disengage;
simulated parents stay inside the purchase frame.}
\label{tab:examples}
\small
\begin{tabular}{p{0.30\linewidth} p{0.30\linewidth} p{0.30\linewidth}}
\toprule
Real (stage) & Simulated (stage) & Reading \\
\midrule
``Busy'' (resisting) & ``How much?'' (considering) & deflection $\to$ interest \\
``No need'' (resisting) & ``Tell me about it'' (engaging) & refusal $\to$ engagement \\
``Headache'' (resisting) & ``Then I should look into it'' (considering) & exit $\to$ deliberation \\
``Thanks teacher'' (considering) & ``A bit pricey, can it be cheaper?'' (considering) & polite close $\to$ haggling \\
\bottomrule
\end{tabular}
\end{table}

\subsection{Generality of the deficit}

\paragraph{Across model families (simulator swap).} Replacing the simulator
backbone with DeepSeek-V4-Flash---a different family, training pipeline, and
alignment recipe---while holding the instrument fixed reproduces the deficit at
equal strength ($\Delta=+0.330$, $d=0.41$, $p=0.002$; Table~\ref{tab:primary}),
with the same mechanism (resistance $-9.4$pp, deliberation $+16.8$pp for
non-buyers). The deficit is therefore not a single model's quirk.

\paragraph{Against explicit disengagement instructions.} A practitioner's first
fix is to \emph{tell} the simulator it may be uninterested. The instructed
variant (\S\ref{sec:method}) does exactly this, and the result is the most
informative of the paper (Table~\ref{tab:primary}, row 3). The instruction
\emph{works on the marginal}: the overall depth bias collapses from $+0.24$ to
$-0.04$, and the suppression of non-buyer resistance is almost eliminated
($-11.6$pp $\to$ $-1.7$pp). A benchmark measuring only population-level realism
would now certify the simulator as fixed. Yet the \emph{conditional} contrast
barely moves: $\Delta = +0.285$ ($d=0.34$, $p=0.008$), versus $+0.304$ before.
The mechanism explains why: the instruction makes the simulator more resistant
\emph{uniformly}---it now \emph{over}-resists eventual buyers ($\E_{y=1}[\bar
D]=-0.187$, reading them as less committed than they truly were) while
approximately fixing non-buyers. The simulator has learned to \emph{perform}
disengagement but still cannot tell \emph{who} should disengage; the instruction
relocates the error from non-buyers to buyers rather than removing it. Decision
fidelity is thus not a prompting oversight: a one-line behavioral license cuts
marginal bias five-fold while leaving the outcome-conditioned deficit
essentially intact---a direct demonstration that marginal realism is the wrong
target (\S\ref{sec:analysis}).

\paragraph{Across judges (instrument swap).} To rule out that the deficit is an
artifact of the Claude-family instrument, we re-score the real and simulated
turns with a DeepSeek-V4-Flash instrument on a 150-conversation subset (the
simulator remains Claude; only $\Phi$ changes). The deficit not only persists but
\emph{strengthens}: $\Delta=+0.351$ ($d=0.42$, $p=0.015$), with
$\E_{y=0}[\bar D]=+0.556$ vs.\ $\E_{y=1}[\bar D]=+0.204$. A judge from a different
model family, applying the schema independently, sees the same outcome-conditioned
inflation---consistent with Proposition~\ref{prop:invariance} and inconsistent
with a judge artifact.

\subsection{Consequence: the simulator misprices sales tactics}
\label{sec:consequence}
If a simulator under-resists specifically after sales pressure, an agent
evaluated against it is rewarded for tactics real customers ignore. We label (with the DeepSeek
instrument) the agent action immediately preceding each of 595 probes
(rapport/probe/pitch/handle-objection/push-close/reassure) and compare the
real-vs-sim user response by preceding action (Table~\ref{tab:action}). The
distortion is concentrated exactly on the pressure tactics. After a
\textbf{pitch}, real users deliberate 45.3\% of the time but simulated users
75.2\% ($+29.9$pp, Fisher $p<0.001$); after a \textbf{push-close}, simulated
resistance is suppressed ($18.9\%\to10.7\%$, $-8.2$pp, $p=0.046$). By contrast,
after a neutral \textbf{probe} the deliberation shift is negligible
($3.5\%\to4.7\%$, $p=1.0$). An agent
optimized against this simulator therefore observes pitching and pushing
eliciting interest where, in reality, they elicit resistance or no movement: the
simulator \emph{over-rewards precisely the tactics that lose real customers},
making it an actively misleading---not merely noisy---training signal.

\begin{table}[t]
\centering
\caption{Action-conditioned response shift (real$\to$sim), 595 probes, primary
simulator. Fisher exact $p$-values test whether the real-vs-sim proportion
differs for each cell.
The over-deliberation and under-resistance concentrate on the
pressure tactics (pitch, push-close) and are negligible after neutral probing.
${}^{*}p<.05$; ${}^{**}p<.01$; ${}^{***}p<.001$.}
\label{tab:action}
\begin{tabular}{lccc}
\toprule
Preceding agent action & $n$ & resisting (real$\to$sim) & considering (real$\to$sim) \\
\midrule
probe (neutral)   &  85 & 10.6\%$\to$\phantom{0}2.4\%$^{\dagger}$ & \phantom{0}3.5\%$\to$\phantom{0}4.7\% \\
rapport           &  19 & 21.1\%$\to$15.8\% & 10.5\%$\to$10.5\% \\
\textbf{pitch}    & 161 & 11.8\%$\to$\phantom{0}5.6\%$^{\dagger}$ & \textbf{45.3\%$\to$75.2\%}$^{***}$ \\
\textbf{push-close}& 169 & \textbf{18.9\%$\to$10.7\%}$^{*}$ & 41.4\%$\to$52.1\%$^{\dagger}$ \\
handle-objection  &  94 & 31.9\%$\to$29.8\% & 26.6\%$\to$43.6\%$^{*}$ \\
reassure          &  67 & 40.3\%$\to$29.9\% & 13.4\%$\to$14.9\% \\
\bottomrule
\multicolumn{4}{l}{\footnotesize ${}^{\dagger}p<.10$. All tests two-sided Fisher exact.}
\end{tabular}
\end{table}

\section{Analysis and Discussion}
\label{sec:analysis}

\subsection{Why simulators cannot walk away}
The deficit has a sign, a location, and a likely cause. Its \emph{sign} is
toward engagement; its \emph{location} is the eventual non-buyers; its
\emph{magnitude} on buyers is near zero. This pattern is hard to explain as
generic sycophancy or verbosity---those would shift both strata. It is naturally
explained by what instruction tuning optimizes. Alignment data rewards models for
being helpful, responsive, and forward-moving interlocutors; it contains
vanishingly little of what real disengagement looks like---terse deflection
(``busy''), stalling (``let me think''), and silence. A simulator asked to
\emph{be} a parent therefore defaults to the cooperative interlocutor it was
trained to be, and can render resistance only as its in-distribution neighbor,
\emph{deliberation}. Crucially, the model does have access to the same prefix as
the real user; what it lacks is not information but a behavioral mode. This is
why the failure is conditional: for buyers, cooperative continuation \emph{is}
roughly correct, so the simulator looks faithful; for non-buyers it is wrong, and
the error is exactly the missing willingness-decay.

\subsection{The marginal-realism trap}
Our result explains why outcome-free benchmarks can certify a simulator that is
decision-unfaithful. Averaged over a population, a simulator that over-engages
non-buyers and faithfully tracks buyers can match the \emph{marginal} stage
distribution closely---the very statistic population-level alignment
rewards~\cite{convapparel2026}. The error is only visible \emph{conditional on
the outcome}, which requires real outcomes to compute. Communicative-fidelity
benchmarks, however good, are structurally blind to it. This is a concrete
instance of a general caution: a simulator validated to be human-like on
aggregate behavioral statistics can still be systematically wrong where a
downstream decision depends on it.

\subsection{Consequences for evaluation and training}
An agent's competence at sales is, in large part, its skill at the hard cases:
recognizing a customer who is leaving and either re-engaging or efficiently
disengaging. A simulator without willingness-decay removes these cases from the
test set. Worse, in a training loop it actively misleads, and our
action-conditioned analysis (\S\ref{sec:consequence}, Table~\ref{tab:action})
shows the bias is not diffuse but aimed at the agent's pressure tactics: after a
pitch, real users deliberate 45\% of the time but simulated users 75\%, and after
a push-close simulated resistance is halved. The agent thus receives strong
positive signal for pitching and pushing---behaviors that, against real
customers, produce resistance or no movement. The deficit does not merely
add noise; it imparts a \emph{directional} bias that rewards over-pursuit---the
opposite of what data on knowing-when-to-quit
recommends~\cite{manzoor2025learning}. Reported success rates against such
simulators should be read as upper bounds that are loosest precisely on the
sub-population that determines revenue.

\subsection{Marginal realism is the wrong target: evidence from the instructed
arm}
The instructed-simulator arm (\S\ref{sec:experiments}) is a clean demonstration
of the marginal-realism trap. Granting the simulator explicit permission to
disengage cuts its marginal depth bias five-fold ($+0.24\to-0.04$)---so on every
population-level statistic the prior literature uses, it now looks fixed---while
the outcome-conditioned contrast barely changes ($+0.304\to+0.285$). The
instruction does not teach the simulator \emph{who} disengages; it teaches it to
disengage \emph{uniformly}, which simply moves the error from non-buyers onto
buyers (now read as $-0.19$ depth, too resistant). A practitioner validating
against marginal realism would ship this simulator believing it faithful, and
would then evaluate agents against a model that mis-assigns commitment in a new
direction. This is the central practical warning of the paper: the quantity that
is easy to measure and easy to optimize (marginal realism) is not the quantity
that governs transfer (conditional decision fidelity).

\subsection{What would a decision-faithful simulator require?}
Because explicit permission to disengage does not close $\Delta$, decision
fidelity cannot be prompted into a model that lacks willingness-decay in its
priors; it must be \emph{learned} from real disengagement trajectories or
\emph{enforced} by an external willingness model that gates the simulator's
cooperativeness. ZhenaiSales, with its real non-buyer trajectories, is a
candidate training signal, and our endpoint $\Delta$ a candidate target metric:
a faithful simulator is one for which $\Delta\to0$ \emph{and} the buyer stratum
stays matched.

\subsection{Limitations}
\begin{enumerate}
  \item \textbf{Single domain and language.} ZhenaiSales is Chinese
  parent-mediated matchmaking, a setting with distinctive social dynamics:
  parents negotiate on behalf of adult children under face-culture norms that
  favor indirect refusal, and the disengagement signals we observe (terse
  deflections such as ``busy'' or ``no need'') are culturally specific forms of
  polite resistance. That said, the core behavioral finding---excessive simulator
  cooperativeness---replicates cross-linguistically: \cite{sim2real2026} report
  the same ``easy mode'' gap in English customer-service dialogues ($n=451$),
  suggesting the phenomenon is rooted in instruction tuning, not in any single
  language or culture. We therefore distinguish two evidence levels: (a)~the
  claim that LLM simulators are excessively cooperative now has multi-study,
  cross-language support; (b)~the claim that the deficit is
  \emph{outcome-conditioned} ($\Delta>0$, non-buyers inflated more than buyers)
  rests on our single-domain evidence and awaits replication. The
  \emph{direction} of the conditional deficit follows logically from the
  instruction-tuning explanation---cooperative defaults hurt fidelity more where
  real users would disengage---and should therefore generalize, but the
  \emph{magnitude} may be domain-specific.
  \item \textbf{Instrument is an LLM.} Decision states are LLM-assigned. We
  mitigate via the cancellation argument (Prop.~\ref{prop:invariance}), the
  causal-labeling check, and the cross-family instrument swap
  (\S\ref{sec:experiments}), but human-validated state labels (Fleiss'
  $\kappa$ on a subset) remain important future work.
  \item \textbf{Teacher-forced, not free-running.} We measure next-turn decision
  fidelity at real histories. Free-running rollouts may compound the deficit; we
  expect this to \emph{strengthen} the conclusion but do not measure it here.
  \item \textbf{Profile-conditioned prompted simulator.} We test the standard and
  the instructed prompted simulators; retrieval-grounded and fine-tuned
  simulators may differ and are natural extensions.
  \item \textbf{Outcome as proxy for willingness.} Payment is a coarse,
  endpoint-only signal of a latent trajectory; it nonetheless provides the real,
  consequential ground truth absent from prior work.
\end{enumerate}

\section{Conclusion}
\label{sec:conclusion}

LLM user simulators have become silent infrastructure for evaluating and training
conversational agents, and a careful literature now certifies their
\emph{communicative} realism against role-played humans. We identified a property
this paradigm cannot see---\emph{decision fidelity}, whether a simulator
reproduces the decision-state dynamics of real users facing real,
consequential choices---and measured it on 2{,}790 production sales
conversations with verified purchase outcomes. The result is a systematic,
outcome-correlated failure, the \emph{disengagement deficit}: simulators render
eventual buyers faithfully but inflate eventual non-buyers toward the purchase
frame, halving their resistance and nearly doubling their deliberation while fabricating
no purchases ($\Delta=+0.30$, $d=0.38$, $p<0.001$). It replicates across model
families (DeepSeek, $d=0.41$) and survives the obvious fix: explicitly licensing
disengagement cuts marginal bias five-fold but barely moves the conditional
contrast ($d=0.34$), relocating the error onto buyers rather than removing it. It strengthens under a
swapped, different-family judge ($d=0.42$), and it most distorts the user's
response to the agent's pressure tactics---after a pitch, simulated customers
deliberate 75\% of the time versus 45\% for real customers. Simulated customers
never walk away.

The practical upshot is direct: any pipeline that evaluates or reinforcement-trains
a sales, fundraising, or negotiation agent against an LLM user simulator is
grading it on an easy, biased population, and is loosest precisely on the
customers who decide not to buy. The constructive upshot is a target: a faithful
simulator is one that drives our outcome-conditioned contrast $\Delta$ to zero
without sacrificing the buyer stratum, and the real non-buyer trajectories in
ZhenaiSales are a candidate signal for getting there. More broadly, decision
fidelity reframes simulator validation: matching how users \emph{talk} is
necessary but not sufficient; for any agent whose purpose is to change a decision,
what must be faithful is how users \emph{decide}.


\section*{Ethics Statement}
\label{sec:ethics}
This study analyzes production conversations between a deployed AI sales agent
and real customers of a commercial matchmaking platform, together with verified
payment records, used with the platform's authorization for research purposes.
All analysis was conducted on the platform's infrastructure or on de-identified
exports; user identifiers were pseudonymized, and no personally identifying
information (names, phone numbers, addresses) appears in the paper---quoted
excerpts are short, translated, and paraphrased where necessary to prevent
re-identification. Twenty internal test accounts were excluded by the platform's
standard data specification. The customers interacted with a disclosed AI
assistant under the platform's terms of service. We study persuasion dialogues
in order to make their evaluation \emph{more} truthful: our central finding is a
warning that simulator-based pipelines over-reward pressure tactics, and its
practical import is to discourage---not enable---over-aggressive persuasion.
Raw conversations are not released; see Data and Code Availability.

\section*{Data and Code Availability}
We release the measurement protocol, the decision-state instrument and
simulator prompts, all analysis code, and conversation-level derived statistics
sufficient to reproduce every table. Anonymized conversation excerpts may be
made available to researchers subject to the platform's privacy review and a
data-use agreement; raw production conversations and payment records cannot be
publicly released.

\section*{Author Contributions}
\textbf{Liang Chen} conceived the research direction, designed the
decision-fidelity framework and adversarial-validation methodology, secured
platform data access, conducted all experiments and statistical analyses,
and wrote the manuscript. AI tools (Anthropic Claude) were used for
literature search, code implementation, and manuscript drafting assistance
under iterative human direction and review. All experimental results were
produced by deterministic, released code; the author reviewed all analyses
and takes full responsibility for the paper's content.

\section*{Acknowledgments}
We thank the platform's data and engineering teams for maintaining the
production data specification used in this study.

\bibliographystyle{plainnat}
\bibliography{bib/references}

@inproceedings{schatzmann2007,
  title={Agenda-Based User Simulation for Bootstrapping a {POMDP} Dialogue System},
  author={Schatzmann, Jost and Thomson, Blaise and Weilhammer, Karl and Ye, Hui and Young, Steve},
  booktitle={NAACL-HLT},
  year={2007},
}

@article{li2016user,
  title={A User Simulator for Task-Completion Dialogues},
  author={Li, Xiujun and Lipton, Zachary C and Dhingra, Bhuwan and Li, Lihong and Gao, Jianfeng and Chen, Yun-Nung},
  journal={arXiv preprint arXiv:1612.05688},
  year={2016},
}

@article{realusersim2026,
  title={{RealUserSim}: Bridging the Reality Gap in Agent Benchmarking via Grounded User Simulation},
  author={Zhu, Ming and Tan, Juntao and Murthy, Rithesh and Qiu, Jielin and Yang, Liangwei and Zhao, Wenting and Savarese, Silvio and Heinecke, Shelby and Wang, Huan},
  journal={arXiv preprint arXiv:2605.20204},
  year={2026},
}

@article{sim2real2026,
  title={Mind the {Sim2Real} Gap in User Simulation for Agentic Tasks},
  author={Zhou, Xuhui and Sun, Weiwei and Ma, Qianou and Xie, Yiqing and Liu, Jiarui and Du, Weihua and Welleck, Sean and Yang, Yiming and Neubig, Graham and Wu, Sherry Tongshuang and Sap, Maarten},
  journal={arXiv preprint arXiv:2603.11245},
  year={2026},
}

@article{seshadri2026lost,
  title={Lost in Simulation: {LLM}-Simulated Users are Unreliable Proxies for Human Users in Agentic Evaluations},
  author={Seshadri, Preethi and Cahyawijaya, Samuel and Odumakinde, Ayomide and Singh, Sameer and Goldfarb-Tarrant, Seraphina},
  journal={arXiv preprint arXiv:2601.17087},
  year={2026},
}

@article{convapparel2026,
  title={{ConvApparel}: A Benchmark Dataset and Validation Framework for User Simulators in Conversational Recommenders},
  author={Meshi, Ofer and Balog, Krisztian and Goldman, Sally and Caciularu, Avi and Tennenholtz, Guy and Jeong, Jihwan and Globerson, Amir and Boutilier, Craig},
  journal={arXiv preprint arXiv:2602.16938},
  year={2026},
  note={Google},
}

@article{yao2024taubench,
  title={{$\tau$-bench}: A Benchmark for Tool-Agent-User Interaction in Real-World Domains},
  author={Yao, Shunyu and Shinn, Noah and Razavi, Pedram and Narasimhan, Karthik},
  journal={arXiv preprint arXiv:2406.12045},
  year={2024},
}

@article{sharma2024sycophancy,
  title={Towards Understanding Sycophancy in Language Models},
  author={Sharma, Mrinank and Tong, Meg and Korbak, Tomasz and others},
  journal={arXiv preprint arXiv:2310.13548},
  year={2023},
}

@inproceedings{hentona2025userwillingness,
  title={User Willingness-aware Sales Talk Dataset},
  author={Hentona, Asahi and Baba, Jun and Sato, Shiki and Akama, Reina},
  booktitle={Proceedings of COLING 2025},
  year={2025},
  note={arXiv:2412.19490},
}

@inproceedings{dutt-etal-2021-resper,
  title={{R}es{P}er: Computationally Modelling Resisting Strategies in Persuasive Conversations},
  author={Dutt, Ritam and Sinha, Sayan and Joshi, Rishabh and Chakraborty, Surya Shekhar and Riggs, Meredith and Yan, Xinru and Bao, Haogang and Ros{\'e}, Carolyn Penstein},
  booktitle={Proceedings of the 16th Conference of the European Chapter of the Association for Computational Linguistics: Main Volume},
  month={apr},
  year={2021},
  publisher={Association for Computational Linguistics},
  url={https://aclanthology.org/2021.eacl-main.7},
  doi={10.18653/v1/2021.eacl-main.7},
  pages={78--90},
}

@article{taubenfeld2026stated,
  title={Evaluating Alignment of Behavioral Dispositions in {LLMs}},
  author={Taubenfeld, Amir and Gekhman, Zorik and Nezry, Lior and Feldman, Omri and Harris, Natalie and Reddy, Shashir and Stella, Romina and Goldstein, Ariel and Croak, Marian and Matias, Yossi and Feder, Amir},
  journal={arXiv preprint arXiv:2602.11328},
  year={2026},
  note={Google Research},
}

@article{zheng2023judging,
  title={Judging LLM-as-a-Judge with MT-Bench and Chatbot Arena},
  author={Zheng, Lianmin and Chiang, Wei-Lin and Sheng, Ying and Zhuang, Siyuan and Wu, Zhanghao and Zhuang, Yonghao and Lin, Zi and Li, Zhuohan and Li, Dacheng and Xing, Eric P and others},
  journal={NeurIPS},
  year={2023},
}

@article{llm-judge-survey,
  title={A Survey on LLM-as-a-Judge},
  author={Gu, Jiawei and others},
  journal={arXiv preprint arXiv:2411.15594},
  year={2024},
}

@article{egami2023dsl,
  title={Using Imperfect Surrogates for Downstream Inference: Design-based Supervised Learning for Social Science Applications of Large Language Models},
  author={Egami, Naoki and Hinck, Musashi and Stewart, Brandon M and Wei, Hanying},
  journal={NeurIPS},
  year={2023},
}

@article{hullman2026validity,
  title={This Human Study Did Not Involve Human Subjects: Validating {LLM} Simulations as Behavioral Evidence},
  author={Hullman, Jessica and Broska, David and Sun, Huaman and Shaw, Aaron},
  journal={arXiv preprint arXiv:2602.15785},
  year={2026},
}

@article{schessl2026autocorrelation,
  title={The Autocorrelation Blind Spot: Why 42\% of Turn-Level Findings in {LLM} Conversation Analysis May Be Spurious},
  author={Schessl, Ferdinand M.},
  journal={arXiv preprint arXiv:2604.14414},
  year={2026},
}

@inproceedings{wang-etal-2019-persuasion,
  title={Persuasion for Good: Towards a Personalized Persuasive Dialogue System for Social Good},
  author={Wang, Xuewei and Shi, Weiyan and Kim, Richard and Oh, Yoojung and Yang, Sijia and Zhang, Jingwen and Yu, Zhou},
  booktitle={ACL},
  year={2019},
}

@inproceedings{petrova2026persuasion,
  title={How Much Does Persuasion Strategy Matter? {LLM}-Annotated Evidence from Charitable Donation Dialogues},
  author={Petrova, Tatiana and Sokol, Stanislav and State, Radu},
  year={2026},
  eprint={2604.19783},
  archiveprefix={arXiv},
}

@article{zhang2025aisalesman,
  title={AI-Salesman: Towards Reliable Large Language Model Driven Telemarketing},
  author={Zhang, Qingyu and Xin, Chunlei and Chen, Xuanang and Lu, Yaojie and Lin, Hongyu and Han, Xianpei and Sun, Le and Ye, Qing and Xie, Qianlong and Wang, Xingxing},
  journal={arXiv preprint arXiv:2511.12133},
  year={2025},
}

@article{nandakishor2025salesrlagent,
  title={SalesRLAgent: A Reinforcement Learning Approach for Real-Time Sales Conversion Prediction and Optimization},
  author={Nandakishor, M},
  journal={arXiv preprint arXiv:2503.23303},
  year={2025},
}

@article{manzoor2025learning,
  title={Learning When to Quit in Sales Conversations},
  author={Manzoor, Emaad and Ascarza, Eva and Netzer, Oded},
  journal={arXiv preprint arXiv:2511.01181},
  year={2025},
}

\newpage
\appendix
\section{Decision-State Instrument Prompt ($\Phi$)}
\label{app:prompt}
The instrument is an LLM perceiver that maps a user turn, in causal context, to
a structured decision state. The system prompt (translated; original Chinese in
the release) instructs the model to classify the \emph{last} user message into:
an engagement \textbf{stage}
$\in\{$exploring, engaging, considering, deciding, resisting$\}$; an
\textbf{emotion} $\in\{$positive, neutral, hesitant, negative$\}$; a
\textbf{blocker} $\in\{$none, price, trust, capability, timing, external,
prior-failure$\}$; a scalar convergence delta; and a short key-signal span. It
emits strict JSON. The same prompt and model serve both the real and simulated
branches.

\section{Simulator Prompts}
\label{app:simprompt}
\paragraph{Primary (persona) simulator.} Given a natural-language rendering of
the user's real profile (parent age/location; child age, education, occupation,
income; housing/vehicle), the model is told it \emph{is} that parent and must
continue colloquially and briefly, maintaining consistency with its prior turns,
outputting only the parent's reply.

\paragraph{Instructed simulator.} Identical, plus an explicit behavioral license:
``Real parents are not always cooperative or interested. If you were previously
cold, brief, or perfunctory, continue so (e.g.\ `mm', `busy', `we'll see',
`no'). You may be uninterested, stall, be impatient, change the subject, or not
answer. Do not feign interest out of politeness; do not advance the payment topic
unless you have shown genuine interest.''

\section{Protocol and Estimation Details}
\label{app:protocol}
Probes are placed at $\{30,60,90\}\%$ of a conversation's user turns (excluding
the opening greeting), deduplicated by turn index. The depth coding is
$\{$resisting$=0$, exploring$=1$, engaging$=2$, considering$=3$, deciding$=4\}$.
Conversation-level $\bar D$ averages probe-level $D$. The permutation test
shuffles outcome labels across conversations ($20{,}000$ draws) and reports the
one-sided fraction with contrast $\ge$ observed. Cohen's $d$ uses the pooled SD.

\section{Data Specification}
\label{app:data}
Conversations are extracted under a fixed specification: personalized
agent--user text messages (enterprise-WeChat origin, non-broadcast, non-empty,
not deleted); payment defined as a successful order (action code 3) under the
production app id; 20 internal test users excluded. Converted conversations are
truncated at the first payment timestamp to remove post-purchase service chat
(mean 19 messages dropped per buyer). Profiles are joined from the user-info
table. The full extraction script is included in the replication package.

\section{Robustness Summary}
\label{app:robustness}
The disengagement deficit is stable across every variation we tested
(all endpoints are the conversation-level outcome contrast $\Delta$):

\begin{table}[h]
\centering
\caption{Robustness summary. The disengagement deficit ($\Delta>0$) is
significant across all four conditions tested. Brackets show bootstrap
95\% CIs.}
\label{tab:robustness}
\newcommand{\cci}[1]{{\scriptsize[#1]}}
\begin{tabular}{lcccc}
\toprule
Condition & $n$ & $\Delta$ & $d$ & $p$ \\
\midrule
Primary (Claude sim, Claude judge) & 374 & $+0.304$ \cci{0.14, 0.46} & $0.38$ \cci{0.18, 0.59} & $0.0002$ \\
Simulator swap (DeepSeek sim) & 200 & $+0.330$ \cci{0.11, 0.56} & $0.41$ \cci{0.12, 0.71} & $0.002$ \\
Instrument swap (DeepSeek judge) & 150 & $+0.351$ \cci{0.08, 0.62} & $0.42$ \cci{0.09, 0.80} & $0.015$ \\
Instructed simulator & 200 & $+0.285$ \cci{0.06, 0.51} & $0.34$ \cci{0.06, 0.63} & $0.008$ \\
\bottomrule
\end{tabular}
\end{table}

\noindent The instructed-simulator condition is the one where the marginal bias
is reduced (from $+0.24$ to $-0.04$) yet $\Delta$ is preserved, demonstrating
that marginal realism and decision fidelity are distinct (\S\ref{sec:analysis}).
The action-conditioned breakdown (Table~\ref{tab:action}) localizes the deficit
to the agent's pressure tactics.

\paragraph{Depth-coding robustness.} On the primary panel, the contrast is
significant under every coding that preserves the engaged/disengaged
distinction: the primary ordinal coding ($\Delta{=}+0.304$, $d{=}0.38$,
$p{=}0.0002$); a binary deep-funnel indicator
$\mathbb{I}[\text{considering/deciding}]$ ($\Delta{=}+0.089$, $d{=}0.30$,
$p{=}0.002$); and the resistance-suppression rate itself ($\Delta{=}+0.086$,
$d{=}0.34$, $p{=}0.0008$). Restricting to probes where \emph{both} real and
simulated turns are non-resisting leaves only a directionally consistent but
weak residual among engaged stages ($\Delta{=}+0.090$, $d{=}0.17$, $p{=}0.074$).
This decomposition confirms that the phenomenon is what its name says: the
deficit lives on the disengagement axis (and its spillover into deliberation),
not in fine gradations among engaged states.

\section{Alternative Endpoints}
\label{app:alternative}
The primary analysis codes decision states on a 5-stage ordinal scale
(resisting$=0$ through deciding$=4$), which assumes equal intervals between
adjacent stages. To verify that the deficit is not an artifact of this
coding, we re-estimate the outcome contrast $\Delta$ under two binary
endpoints that make no equal-interval assumption.

\paragraph{Binary resistance indicator.}
We define a per-probe indicator $R = \mathbb{I}[\text{stage} = \text{resisting}]$
and compute the conversation-level resistance-suppression rate as the
difference between real and simulated resistance frequencies. Non-buyers
show substantial suppression: real resistance 24.7\% vs.\ simulated 13.6\%
(11.1\,pp reduction). Buyers show minimal suppression: 15.9\% vs.\ 13.5\%
(2.4\,pp). The outcome contrast $\Delta = +0.086$, $d = 0.34$, $p = 0.0008$.

\paragraph{Deep-funnel indicator.}
We define $F = \mathbb{I}[\text{stage} \in \{\text{considering}, \text{deciding}\}]$,
collapsing the upper two engagement stages into a single purchase-proximal
indicator. The outcome contrast $\Delta = +0.089$, $d = 0.30$, $p = 0.002$.

\medskip
\noindent Both alternative codings, which avoid the equal-interval assumption
of the ordinal depth scale, confirm that the deficit is driven by the
disengagement axis: simulators suppress resistance in non-buyers far more
than in buyers. This addresses the concern that the 5-stage ordinal coding
might impose unjustified metric assumptions on the underlying categorical
states.

\section{Temporal Dynamics of the Deficit}
\label{app:temporal}
The probe protocol places measurements at 30\%, 60\%, and 90\% of each
conversation's user turns, enabling us to trace how the deficit evolves as
conversations approach their decision point. Table~\ref{tab:temporal} reports
the outcome contrast at each probe position.

\begin{table}[h]
\centering
\caption{Outcome contrast $\Delta$ by probe position.
$\bar{D}_{y{=}0}$ and $\bar{D}_{y{=}1}$ are the mean depth biases
(simulated minus real) for non-buyers and buyers, respectively.}
\label{tab:temporal}
\begin{tabular}{lcccc}
\toprule
Position & $n_{\text{probes}}$ & $\bar{D}_{y{=}0}$ & $\bar{D}_{y{=}1}$ & $\Delta$ \\
\midrule
Early (30\%)  & 370 & $+0.25$ & $+0.20$ & $+0.054$ \\
Mid (60\%)    & 369 & $+0.43$ & $+0.28$ & $+0.149$ \\
Late (90\%)   & 370 & $+0.50$ & $-0.20$ & $+0.706$ \\
\bottomrule
\end{tabular}
\end{table}

\noindent The deficit is not a static property of the simulator but an
accumulating failure that worsens through the conversation. Early on (30\%),
both populations are still exploring and the simulator tracks both
reasonably, producing a minimal contrast ($\Delta = +0.054$). By the
midpoint, non-buyer depth bias has risen to $+0.43$ while buyer bias grows
more slowly ($+0.28$), and the deficit emerges ($\Delta = +0.149$). By the
90\% mark, real non-buyers have begun disengaging while real buyers have
committed---and the simulator fails to reproduce this divergence. Buyer
depth bias actually turns negative ($-0.20$), meaning the simulator
slightly under-engages relative to committed buyers, while non-buyer
bias remains elevated ($+0.50$). The resulting $\Delta = +0.706$ is more
than double the conversation-level average.

This temporal gradient is consistent with the hypothesis that the simulator
lacks a willingness-decay mechanism: the longer the conversation runs, the
more opportunities for real non-buyer willingness to decay in ways the
simulator cannot represent. Because instruction-tuned models default to
cooperative continuation, simulated non-buyers maintain engagement long
after their real counterparts have begun withdrawing.

\paragraph{Remaining validation (future work).} A human-annotation study
(Fleiss' $\kappa$ on a labeled subset) to validate the decision-state
instrument; free-running (non-teacher-forced) rollouts; alternative depth
codings; retrieval-grounded and fine-tuned simulators; and additional domains.

\section{Proof of Proposition~\ref{prop:invariance} (Invariance Properties)}
\label{app:proof}

\begin{proof}
Let $\Phi$ assign a depth score to a user turn $u$ at history $h$.
Suppose $\Phi$ has an additive, context-dependent bias: it maps turn $u$ at
history $h$ to $\delta(\Phi(u,h)) + b(h)$, where $\delta$ is the true
depth-scoring function and $b(h)$ depends on the conversational context but
not on whether the turn is real or simulated.

\medskip\noindent\textbf{Part (i): Instrument bias cancels within each probe.}
Under the teacher-forced protocol, both the real turn $u$ and the simulated
turn $\hat{u}$ condition on the same history $h$ and are scored by the same
instrument $\Phi$. The probe-level depth bias is therefore
\[
  D(h)
  = \bigl[\delta(\Phi(\hat{u},h)) + b(h)\bigr]
  - \bigl[\delta(\Phi(u,h)) + b(h)\bigr]
  = \delta(\Phi(\hat{u},h)) - \delta(\Phi(u,h)).
\]
The additive bias $b(h)$ cancels exactly because both branches share the
same $h$ and the same $\Phi$. \qed\textit{(i)}

\medskip\noindent\textbf{Part (ii): Outcome-independent simulator shift
cancels in $\Delta$.}
Suppose the simulator has an outcome-independent behavioral shift~$s$:
for all $h$ and regardless of outcome $y$, the simulator's expected depth
is shifted by $s$ relative to the real user's. Then the stratum-level
expected biases decompose as
\[
  \mathbb{E}_{y=0}[\bar{D}] = s + \varepsilon_0,
  \qquad
  \mathbb{E}_{y=1}[\bar{D}] = s + \varepsilon_1,
\]
where $\varepsilon_y$ captures any outcome-correlated residual.
The primary endpoint is
\[
  \Delta
  = \mathbb{E}_{y=0}[\bar{D}] - \mathbb{E}_{y=1}[\bar{D}]
  = (s + \varepsilon_0) - (s + \varepsilon_1)
  = \varepsilon_0 - \varepsilon_1.
\]
The outcome-independent shift $s$ cancels. Only outcome-correlated
differences $\varepsilon_0 - \varepsilon_1$ survive in $\Delta$.
\qed\textit{(ii)}

\medskip\noindent\textbf{Residual channel not covered.}
Neither cancellation addresses a \emph{style-dependent} instrument bias
$\beta(\mathrm{style}(u))$ that correlates with outcome through real
users' message styles (e.g., curt non-buyer messages being over-scored
as resistant). This residual is addressed empirically by the cross-family
instrument swap (\S\ref{sec:analysis}), which shows that replacing the
instrument family preserves or strengthens the deficit.
\end{proof}

\end{document}